\documentclass[twocolumn]{article}
\usepackage{amsmath,amssymb,amsbsy,amsthm,algorithm,algorithmic,mathtools, graphicx,hyperref}%,pxfonts}

\let\emptyset\varnothing
\let\epsilon\varepsilon
\let\phi\varphi

\def\cl{\operatorname{cl}}
\def\min{\operatorname{min}}
\def\mod{\operatorname{mod}}
\def\X{{\cal X}}
\def\N{\mathbb N}
\def\S{\mathcal S}
\def\R{\mathbb R}

\def\C{\mathcal C}
\def\E{{\bf E}}
\def\v{{v}}

\def\x{{\mathbf{x}}}
\def\y{{\mathbf{y}}}
\def\z{{\mathbf{z}}}

\def\U{{\mathbf{U}}}

\def\nperp{{\mathrlap\perp{\,\setminus}}}

\def\I{{{}^{{s}}}\hskip-2ptI}
\def\hI{\widehat{{{}^{{s}}}\hskip-2ptI}}

\def\-as{\text{-a.s.}}

\newtheorem{theorem}{Theorem}
\newtheorem{definition}{Definition}
\newtheorem{lemma}{Lemma}
\newtheorem{proposition}{Proposition}

\newtheorem*{test}{TEST}
{\theoremstyle{definition}\newtheorem{remark}{Remark}}%\newtheorem{remark}{Remark}

\begin{document}
%\title{Clustering processes with respect to dependence}
\title{Independence clustering (without a matrix)}
\author{Daniil Ryabko % \\ {\em INRIA } 
  \\ {\tt daniil@ryabko.net} }
 %\\ {\em DRAFT PAPER pls do not distribute}}
\date{}
%\twocolumn[
%\icmltitle{Clustering processes}
%\icmlauthor{ Daniil Ryabko}{daniil@ryabko.net} 
%\icmladdress{INRIA Lille, 40 avenue Halley,
%59650 Villeneuve d'Ascq, France}
%\icmlkeywords{}
%]
\maketitle

\begin{abstract} The  independence clustering problem is considered in the following formulation: given a set $S$ of random variables,  it is required to find the finest partitioning $\{U_1,\dots,U_k\}$ of  $S$ into clusters  such that the clusters $U_1,\dots,U_k$ are mutually independent. Since mutual independence is the target, pairwise similarity measurements are of no use, and thus traditional clustering algorithms are inapplicable.   The distribution of the random variables in $S$ is, in general, unknown, but a sample  is available. 
Thus, the problem is cast in terms of time series.  Two forms of sampling are considered: i.i.d.\ and stationary  time series, with the main emphasis being on the latter, more general, case. A consistent, computationally tractable algorithm for each of the settings is proposed, and a number of
open directions for further research are outlined. 
\end{abstract}

% \begin{keywords}
% \end{keywords}

\section{Introduction}
Many applications face the situation where a set $S=\{\x_1,\dots,\x_N\}$ of samples has to be divided into clusters in such a way that inside each cluster the samples are dependent, but the  clusters between themselves  are as independent as possible. Here each $\x_i$ may itself be a sample, or a time series $\x_i=X^i_1,\dots,X^i_n$. For example, in financial applications, $\x_i$ can be a series of recordings of prices of a stock $i$ over time. The goal is to find the segments of the market such that different segments evolve independently, but within each segment the prices are mutually informative \cite{mantegna1999hierarchical,marti2016clustering}. In biological applications, each  $\x_i$ may be a DNA sequence, or may represent gene expression data \cite{zhou2004gene,priness2007evaluation}, or, in other applications, an fMRI record \cite{benjaminsson2010novel,kolchinsky2014multi}. 

The staple approach to this problem in applications is to construct a matrix of (pairwise) correlations  between the elements, and use traditional clustering methods, e.g., linkage-based methods or $k$ means and its variants, with this matrix \cite{mantegna1999hierarchical,marti2016clustering,marrelec2015bayesian}. If mutual information is used, it is used as a (pairwise) proximity measure between individual inputs, e.g.~\cite{kraskov2005hierarchical}.

We remark that pairwise independence is but a surrogate for (mutual) independence, and, in addition, correlation is   a surrogate for pairwise independence. There is, however,  no need to resort to surrogates unless forced to do so by statistical or computational hardness results. 
%Furthermore, while it is appealing to be able to use readily available clustering algorithm for the independence clustering problem, this  is by no means  a primary goal.

We therefore propose to reformulate the problem from the first principles, and then proceed to show that it is indeed solvable both statistically and  computationally~--- but calls for completely different algorithms.
The  %problem 
formulation proposed is as follows.

{\em
Given a  set $S=(\x_1,\dots,\x_N)$ of random variables, it is required to find the finest partitioning $\{U_1,\dots,U_k\}$ of  $S$ into clusters  such that the clusters $U_1,\dots,U_k$ are mutually independent.
}

To the author's knowledge, this problem in its full generality has not been addressed before. The formulation appears  in the work \cite{bach2003beyond} (which is devoted to optimizing a generalization of the ICA objective),
 but the case studied there (tree-structured dependence within the clusters) is  a one that admits a solution in terms of pairwise measurements of mutual information. 
%This latter case allows for a solution based on pairwise measurements of  mutual information.

Note that in the fully general case pairwise measurements are useless, as are, furthermore, bottom-up  (e.g., linkage-based) approaches. Thus, in particular, a proximity matrix cannot be used for the analysis. Indeed, it is easy to construct examples in which any pair or any small group of elements are independent, but are dependent when the same group is considered jointly with more elements. For instance, consider a group of  Bernoulli 1/2-distributed random variables  $\x_1,\dots,\x_{N+1}$, where $\x_1,\dots,\x_N$
are i.i.d.\ and $\x_{N+1}=\sum_{i=1}^N\x_i \operatorname{mod} 2$. Note that any $N$ out of these $N+1$ random variables are i.i.d., but together the $N+1$ are dependent.  Add then two more groups like this, say, $\y_1,\dots,\y_{N+1}$ and $\z_1,\dots,\z_{N+1}$ that have the exact same distribution, with the groups of $\x$, $\y$ and $\z$ mutually independent. Naturally, these are the three clusters we would want to recover. However, if we try to cluster the union of the three, then any algorithm based on pairwise correlations will return an essentially arbitrary result. What is more, if we try to find {\em clusters} that are pairwise independent, then, for example, the clustering $\{(\x_i,\y_i,\z_i)_{i=1..N}\}$ of the input set into $N+1$ clusters  appears correct, but, in fact, the resulting clusters are dependent. Of course, real-world data does not come in the form of summed up Bernoulli variables, but this simple example shows that considering independence of small subsets may be very misleading.

We separate the considered problem into the algorithmic and the statistical part. This is done by first considering the problem assuming the joint distribution of all the random variables is known, and is accessible via an oracle. Thus, the problem becomes computational. A simple, computationally efficient algorithm is proposed for this case. 
% In the proposed formulation, since the dependence of the random variables may be arbitrary complex, it is not  immediately apparent that a computationally tractable solution exists, even in the case when the distributions of all the random variables in question is known.  Thus, we first consider the problem assuming an oracle access to the joint distribution. For this formulation  a  consistent algorithm is proposed that runs in $O(kN^2)$ time. 
 We then proceed to the time-series formulations: the distribution of $(\x_1,\dots,\x_N)$ is unknown, but a sample $(X^1_1,\dots,X^N_1),\dots,(X^1_n,\dots,X^N_n)$ is provided, so that  $\x_i$ can be identified with the time series $X^i_1,\dots,X^i_n$. 
The sample may be either independent and identically distributed (i.i.d.), or, in a more general formulation, stationary. As one might expect, relying on the existing statistical machinery, the case of known distributions can be directly extended to the case of i.i.d.\ samples. Thus, we show that it is possible to replace the oracle access with statistical tests and estimators, and then use the same algorithm as in the case of known distributions. 
The general case of stationary samples turns out to be much more difficult, in particular because of a number of strong impossibility results. In fact, it is challenging already to determine what is possible and what is not from the statistical point of view. In this case, it is not possible to replicate the oracle access to the distribution, but only its weak version that we call {\em fickle oracle}.
We find that, in this case, it is only possible to have a consistent algorithm for the case of known $k$.  An algorithm that has  this property is constructed. This algorithm is computationally  feasible when the number of clusters $k$ is small, as its complexity is $O(N^{2k})$.  Besides, a measure  of information divergence is proposed for time-series distributions that may be of independent interest, since it can be estimated consistently without any assumptions at all on the distributions or their densities (the latter may not exist).

\noindent {\bf The  main results} of this work are theoretical. The goal is to determine, as a first step, what is possible and what is not from both statistical and computational points of view. The main emphasis is placed on highly dependent time series, as warranted by the applications cited above. Detailed experimental investigations of the proposed methods are left  for future work. The contribution can be summarized as follows:
\begin{itemize}
% \item a novel formulation of the independence clustering problem, which is simple and natural for a wide range of applications;
 \item a consistent, computationally feasible algorithm for the case of known distributions, unknown number of clusters, along with its extension to the case of unknown distributions and i.i.d.\ samples;
% \item a modification of the above which is consistent for unknown distributions under i.i.d. sampling;
 \item an algorithm that is consistent under stationary ergodic sampling with arbitrary,  unknown distributions, but with a known number $k$ of clusters;
 \item an impossibility result concerning clustering under stationary ergodic sampling with $k$ unknown;
 \item an information divergence measure for stationary ergodic time-series distributions along with its estimator that is consistent without any extra assumptions.
% \item an array of open problems and exciting directions for future work.
\end{itemize}
In addition, an array of open problems and exciting directions for future work is proposed.

\noindent {\bf Related work.} Besides the work on independence clustering mentioned above, it is worth pointing out the relation to some other problems. First, the proposed problem formulation can be viewed as a  Bayesian-network learning problem: given an unknown network, it is required to split it into independent clusters. In general, learning a Bayesian network is NP-hard \cite{chickering1996learning}, even for rather restricted classes of networks (e.g., \cite{meek2001finding}). Here the problem we consider is much less general, which is why it admits a polynomial-time solution. A related clustering problem, proposed in \cite{Ryabko:10clust} (see also \cite{Khaleghi:15clust}) is clustering time series with respect to distribution. Here, it is required to put two time series samples $\x_1,\x_2$ into the same cluster if and only if their distribution is the same. Similar to the independence clustering introduced here, the problems admits a consistent algorithm if the samples are i.i.d.\ (or mixing) and the number of distributions (clusters) is unknown, and in the case of stationary ergodic samples if and only if $k$ is known. 
 A problem that is seemingly very much related, but, in fact, is rather different, is clustering with respect to mutual information. Here one seeks to find a clustering that maximizes the mutual information between the cluster labels and the input variables \cite{heading1991unsupervised}. While using similar terminology, this is in fact a very different problem, as here the inputs are seen as (independent) realization of the r.v. corresponding to the cluster labels; this leads to a completely different goal.

\noindent{\bf Organization.} The next section introduces the setup, defining the problem and consistency of algorithms. Section~\ref{s:know} presents the algorithm for the case of known distributions and proves its consistency. Section~\ref{s:iid} extends the latter algorithm to the case when the distribution is not known but an   i.i.d.\ sample of it is provided. Section~\ref{s:st} is devoted to the case when the samples  are not i.i.d.\ but stationary, presents impossibility results for this case, and a consistent algorithm for the case of known $k$. Section~\ref{s:ext} presents a number of possible extensions of the proposed methods, as well as an array of  open questions and research directions for future work.

\section{Set-up and preliminaries}
A set $S:=\{\x_1,\dots,\x_N\}$ is given, where we will either assume  that the joint distribution of $\x_i$  is known, or else that the distribution is unknown but a sample  $(X_1^1,\dots,X_n^1),\dots, (X_1^N,\dots,X_n^N)$ is given. In the latter case, we identify each $\x_i$ with the sequence (sample) $X^i_1,\dots,X^i_{n}$, or $X^i_{1..n}$ for short, of length $n$. The lengths of the samples are the same only for the sake of notational convenience; it is easy to generalize all algorithms to the case of different sample lengths $n_i$, but the asymptotic would then be with respect to $n:=\min_{i=1..N}n_i$.
It is assumed that $X^i_j\in\X:=\R$ are real-valued, but extensions to more general cases are straightforward. 

For random variables $A,B,C$ we write $(A\perp B)|C$  to say that $A$ is conditionally independent of $B$ given $C$, and $A\perp B\perp C$ to say that $A,B$ and $C$ are mutually independent.

\begin{definition}[Ground-truth clustering]
 The (unique up to a permutation) partitioning $\U:=\{U_1,\dots,U_k\}$ of the set $S$ is called the {\em ground-truth clustering} if $U_1,\dots,U_k$ 
are mutually independent ($U_1\perp\dots\perp U_k$) and  no refinement of $\U$ has this property. 
\end{definition}
\begin{definition}[Consistency] A clustering algorithm is consistent if it outputs the ground-truth clustering, and it is asymptotically consistent if it outputs the ground-truth clustering from some $n$ on with probability~1.
\end{definition}

For a discrete $A$-valued r.v.\ $X$ its Shannon entropy is defined as $H(X):=\sum_{a\in A}-P(X=a)\log P(X=a)$, letting $0\log0=0$.
For a distribution with a density $f$ its {\em (differential) entropy} is defined as $H(X)=:-\int f(x)\log f(x)$. 
For two random variables $X,Y$ their {\em mutual information} $I(X,Y)$ is defined as $I(X,Y)= H(X)+H(Y)-H(X,Y)$.
For discrete random variables, as well as for continuous ones with a density,  $X\perp Y$ if and only if $I(X,Y)=0$; see, e.g., \cite{Cover:06}. Likewise, $I(X_1,\dots,X_m)$ is defined as $\sum_{i=1..m}H(X_i) - H(X_1,\dots,X_m)$.

For the sake of convenience, in the next two sections we make the following assumption. 
{\\\noindent\bf Assumption 1.} (densities) All distributions in question have densities, and these densities are bounded away from zero on their support. % with bounded support. 
% and the corresponding entropies are finite.

However, it will be shown in Sections~\ref{s:st},\ref{s:ext} that this assumption can be gotten rid of as well.

\section{Known distributions}\label{s:know}
As with any statistical problem, it is instructive to start with the case where the (joint) distribution of all the random variables in question is known. Finding  out what can be done and how to do it in this case helps us to set the goals for the (more realistic) case of unknown distributions. 

Thus, in this subsection, $\x_1,\dots,\x_N$ are not time series, but random variables whose joint distribution is known to the statistician. The access to this distribution is via an oracle; specifically, our oracle will provide answers to the following questions about mutual information (where, for convenience, we assume that the mutual information with the empty set is 0): 
\begin{test}[\bf By oracle]
 Given sets of random variables $A,B,C,D\subset\{\x_1,\dots,\x_N\}$ answer whether $I(A,B)>I(C,D)$.
\end{test}

As we will show, such an oracle is sufficient to find the correct clustering. 
\begin{remark}[\bf Conditional independence oracle]\label{r:or}
Equivalently, one can consider an oracle that answers {\em conditional independence} queries of the form  $(A\perp B)|C$. The definition above is chosen for the sake of continuity with the next section, and it also makes the algorithm below more intuitive. Note, however, that in order to test conditional independence statistically one does not have to use mutual information, but may resort to any other divergence measure instead. 
\end{remark}

The proposed algorithm (see the pseudocode listing below) works as follows. It attempts to split the input set recursively  into two independent clusters, until it is no longer possible.
To split a set in two, it starts with putting one element $\x$ from the input set $S$ into a candidate cluster $C:=\{\x\}$, and measures its mutual information $I(C,R)$ with the rest of the set, $R:=S\setminus C$. If $I(C,R)$ is already 0 then we have split the set into two independent clusters and can stop. Otherwise, the algorithm then takes the elements out of $R$ one by one {\em without replacement} and each time looks  whether $I(C,R)$ has 
decreased. Once such an element is found,  it is moved from $R$ to $C$ and {\em the process is restarted} from the beginning with $C$ thus updated. 
Note that, if we have started with $I(C,R)>0$, then taking elements out of $R$ without replacement we eventually should find a one that decreases $I(C,R)$, since $I(C,\emptyset)=0$ and $I(C,R)$ cannot increase in the process.
%; otherwise, it is not returned to $R$ but put into a  hold set $M$ instead. 
%If none such element is found, we have split the set. 
%Otherwise, we have to repeat the whole procedure again (put $M$ back into $R$), until either all elements are moved into $C$ or none can be moved any more.

%In this section we consider the case when, for each given $i$, the elements $X^i_1,\dots,X^i_{n_i}$ are independent and identically distributed. Thus, all the dependence is between the sequences rather than within. In other words, we  have to cluster the random variables $\x_1,\dots,\x_k$ with respect to dependence given i.i.d.\ samples.
\begin{algorithm}[h]
\caption{CLIN: cluster given a test for mutual information}
\label{alg:1}
\begin{algorithmic}
\STATE {INPUT: The set $S$.}
\STATE {$(C_1,C_2):=Split(S)$}
\IF {$C_2\ne\emptyset$}
 \STATE{ Output:$CLIN(C_1),CLIN(C_2)$}
 \ELSE {\STATE{ Output: $C_1$}}
\ENDIF
{\\\bf Function {Split}}{(Set $S$ of samples)}%$\{\x_1,\dots,\x_m\}$} 
\STATE {Initialize: $C:=\{\x_1\},$ $R:=S\setminus C$;} %R:= % $S:= \{\x_2,\dots,\x_m\}$ }%  $S\setminus C$ }%\{\x_2,\dots,\x_m\}$}
\WHILE {TEST($I(C;R)> 0$)}
\FOR {each $\x\in R$}
% \STATE { $O:=O\backslash M$}
 \IF {TEST($I(C;R)>I(C;R\setminus\{\x\})$)}
   \STATE{move $\x$ from $R$ to $C$}
  \STATE{break the {\bf for} loop}
   \ELSE \STATE{move $\x$ from $R$ to $M$}
 \ENDIF
 \ENDFOR
 \STATE $M:=\{\},\  R:=S\setminus C$;
\ENDWHILE 
\STATE{Return($C$,$R$)}
%\IF {$R\ne\emptyset$} \STATE Return($C,R$) \ELSE \STATE Return($C$)\ENDIF
\\\bf END function 
\end{algorithmic}
\end{algorithm}
\begin{theorem} The algorithm CLIN outputs the correct clustering using at most $2kN^2$ oracle calls. % (provided the oracle does not lie).
\end{theorem}
\begin{proof}
 We shall first show that the procedure for splitting a set into two indeed splits the input set into two independent sets, if  and only if such two sets exist. First, note that if  $I(C,S\setminus C)=0$ then $C\perp R$ and the function terminates. In the opposite case, when  $I(C,S\setminus C)>0$, by removing an element from $R:=S\setminus C$, $I(C,R)$ can only decrease (indeed, $h(C|R)\le h(C|R\setminus\{x\})$ by information processesing inequality, e.g. \cite{Cover:06}). Eventually when all elements are removed, $I(C,R)=I(C,\{\})=0$, so there must be an element $\x$ removing which decreases $I(C,R)$. 
When such an element $\x$ found it is moved to $C$. Note that,  in this case, indeed $\x\nperp C$. However, it is possible that removing an element  $\x$ from $R$ does not reduce $I(C,R)$, yet $\x\nperp C$. This is why the \texttt{while} loop is needed, that is, the whole process has to be repeated until no elements can be moved to $C$. By the end of each \texttt{for} loop, we have either found at least one element to move to $C$, or we have assured that $C\perp S\setminus C$ and the loop terminates. Since there are only finitely many elements in $S\setminus C$ to begin with, the \texttt{while} loop eventually terminates. 
Moreover, each of the two loops (\texttt{while} and \texttt{for}) terminates in at most $n$ iterations.

Finally,  notice that if $(C_1,C_2)\perp C_3$ and $C_1\perp C_2$ then also $C_1\perp\C_2\perp C_3$, which means that by repeating the Split function recursively we find the correct clustering.

From the above, the bound on the number of oracle calls is easily obtained by  direct calculation.
\end{proof}
\section{I.I.D.\ sampling}\label{s:iid}
In this section we assume that the distribution of $(\x_1,\dots,\x_N)$ is not known, but an i.i.d.\ sample 
$(X^1_1,\dots,X^N_1),\dots,(X^1_n,\dots,X^N_n)$ is provided. We identify $\x_i$ with the (i.i.d.) time series $X^i_{1..n}$.

The case of i.i.d.\ samples is not much different from the case of a known distribution. The only difference is that we need to replace the TEST oracle with  (nonparametric) statistical tests.  What we need is, first, a test for independence, which is needed to replace the oracle call TEST($I(C,R)>0$) in the \texttt{while} loop.  Such a test can be found in \cite{gretton2010consistent}. 
Second, we need an estimator of mutual information $I(X,Y)$, or, which is sufficient, for entropies, but with a rate of convergence. 
If we know that the rate of convergence is asymptotically bounded by, say, $t(n)$, we can take any $t'(n)\to0$ such that $t(n)=o(t'(n))$ and   decide our inequality as follows:  
if $\hat I(C;R\setminus\{\x\})<\hat I(C;R)-t'(n)$ then say that $I(C;R\setminus\{\x\})<I(C;R)$. The required rates of convergence, which are of order $\sqrt{n}$ under Assumption~1, can be found in \cite{beirlant1997nonparametric}.
% with shrinking confidence bounds,  in order to construct a test to replace TEST($I(C;R)>I(C;R\setminus\{\x\})$). Having some confidence bounds for the estimator we can decide the answer to the test: if the $\hat I(C;R\setminus\{\x\})$ is outside the confidence interval for $I(C;R)$.
% Both the independence test and the (exponential) bounds can be found in \cite{gretton2010consistent}. 
%  For the former, we can use, for example, a test from \cite{gretton2010consistent}, which makes but a finite number of errors a.s.\ without any assumptions on the distributions.  For the latter, it is enough to have a consistent estimator $\hat I$ for mutual information, and replace the condition above with $\hat I(C;R)>\hat I(C;R\setminus\{\x\})$.  Various estimators of this kind are known; for example, \cite{gyorfi1987density} provides a histogram-based estimator, which is a.s.\  consistent under Assumption~1.

Given the result of the previous section, it is clear that if the oracle is replaced by the tests described, then  CLIN is a.s.\ consistent. Thus, we have demonstrated the following. 
\begin{theorem}\label{th:2}
 Under Assumption~1, there is an asymptotically consistent algorithm for independence clustering with i.i.d.\ sampling.
\end{theorem}

\begin{remark}[\bf Necessity of the assumption]
 Note that the independence test of \cite{gretton2010consistent} is actually distribution-free, meaning that it does not need Assumption~1. Since the mutual information is defined in terms of densities, if we want to completely get rid of  Assumption~1, we would need to use some other measure of dependence for the test.  One such measure is defined in the next section already for the general case of process distributions. 
However, the rates of convergence of its empirical estimates under i.i.d.\ sampling remain to be studied.
%Thus, we can say that the existence result of {\em Theorem~\ref{th:2} holds without Assumption~1 as well.}   
\end{remark}
\begin{remark}[\bf Estimators vs.\ tests]\label{r:et}
As noted in Remark~\ref{r:or} above, the tests we are using are, in fact, tests for   (conditional) independence: the test $I(C;R)>I(C;R\setminus\{\x\})$ can be replaced with a test for $(C\perp\{\x\}|R\setminus\{\x\})$. Conditional independence can be tested directly, without estimating $I$ (see, for example \cite{zhang2011kernel}), potentially allowing for tighter performance guarantees under more general conditions. 
%; we leave further investigations in this direction, as well as finite-sample performance guarantees in general, for future work. 
\end{remark}

\section{Stationary sampling}\label{s:st}
We now enter the hard mode. The case of stationary sampling presents numerous obstacles, some of which are, in fact, theoretical impossibility results: there are (provably) no rates of convergence, no independence test, and zero mutual information rate does not guarantee independence. Besides, some simple-looking questions regarding the existence of consistent tests, which indeed have simple answers in the i.i.d.\ case,  remain open in the stationary ergodic case. 

Despite all this, it is possible to obtain a computationally feasible asymptotically consistent independence clustering algorithm, although only for the case when the number of clusters is known. This parallels the situation of clustering according to the distribution \cite{Ryabko:10clust,Khaleghi:15clust}.

Thus, in this section we assume that the distribution of $(\x_1,\dots,\x_N)$ is not known, but a jointly stationary ergodic sample 
$(X^1_1,\dots,X^N_1),\dots,(X^1_n,\dots,X^N_n)$ is provided. We identify $\x_i$ with the (stationary ergodic) time series $X^i_{1..n}$.
In this section we {\em drop Assumption~1}; in particular, {\em densities do not have to exist}.

We start with some preliminaries about stationary and stationary ergodic processes (including the definitions thereof), followed by impossibility results for the problem at hand, and concluding with the algorithm proposed for the case of a known number of clusters.

\subsection{Preliminaries: stationary ergodic processes}
{\noindent\bf Stationary, ergodicity, information rate.}
 (Time-series) distributions, or  processes, are measures on the space $(\X^\infty,\mathcal F_{\X^\infty})$, where $\mathcal F_{\X^\infty}$ is the 
Borel sigma-algebra of $\X^\infty$. When talking about joint distributions of $N$ samples, 
it is distributions on the space $((A^N)^\infty,\mathcal F_{(A^N)^\infty})$ that are referred to, and this distinction will be often left implicit.

For a sequence $\x\in A^n$ and a set $B\in \mathcal B$ denote $\nu(\x,B)$
the {\em frequency} with which the sequence $\x$ falls in the set~$B$.
% {\scriptsize
%  \begin{multline*} 
%  \nu(\x,B):=\\ \left\{ \begin{array}{rl}  {1\over n-|B|+1}\sum_{i=1}^{n-|B|+1}
%  I_{\{(X_i,\dots,X_{i+|B|-1})\in B\}} & \text{ if }n\ge |B|, \\
%  0 & \text{ otherwise.}\end{array}\right.
%  \end{multline*}
% }
%  
% 
A process $\rho$ is {\em stationary}  if 
$
\rho(X_{1..|B|}=B)=\rho(X_{t..t+|B|-1}=B)
$
 for any measurable $B\in \X^*$ and $t\in\N$. We further abbreviate   $\rho(B):=\rho(X_{1..|B|}=B)$.
%%Denote  $\S$ the set of all stationary  processes on $A^\infty$.
A stationary process $\rho$ is called {\em (stationary) ergodic} if
   the frequency of occurrence of each measurable
$B\in\X^*$ in a sequence $X_1,X_2,\dots$ generated by $\rho$ tends to its
a priori (or limiting) probability a.s.: 
$
\rho(\lim_{n\rightarrow\infty}\nu(X_{1..n},B)= \rho(B))=1.
$ 
By virtue of the ergodic theorem,  this definition can be shown to be equivalent 
 to the more standard definition of stationary ergodic processes given in terms of shift-invariant sets  \cite{Shields:98}. 
 Denote $\mathcal S$ and  $\mathcal E$ the sets  of all stationary and stationary ergodic processes correspondingly.

The  {\bf ergodic decomposition} theorem for stationary processes  (see, e.g., \cite{Gray:88}) states that
 any stationary process can be expressed as a mixture of stationary ergodic
processes. That is, a stationary process $\rho$ can be envisaged as first selecting a stationary ergodic distribution according to some measure $W_\rho$ over the set of all such distributions, and then using this ergodic distribution to generate the sequence.
%In other words, a stationary process can be envisaged as first selecting a stationary
%ergodic process with respect to some distribution, and then using it (to generate a sequence of outcomes). 
More formally, for any $\rho\in\S$ there is a measure $W_\rho$ on $(\S,\mathcal F_\S)$, such 
that $W_\rho(\mathcal E)=1$, and $\rho(B)=\int d W_\rho(\mu)\mu(B)$, for any $B\in\mathcal F_{\X^\infty}$.

%{\noindent\bf Entropy, information.} 
For a stationary time series $\x$, its {\em $m$-order entropy} $h_m(\x)$ is defined as $\E_{X_{1..m-1}}h(X_{m}|X_{1..m-1})$ (so the usual Shannon entropy is the entropy of order 0). By stationarity, the limit $\lim_{m\to\infty}h_m$ exists and equals $\lim_{m\to\infty}{1\over m}h(X_{1..m})$ (see, for example, \cite{Cover:06} for more details).
 This limit is called {\em entropy rate} and is denoted $h_\infty$. For $l$ stationary processes $\x_i=(X^i_1,\dots,X^i_n,\dots)$, $i=1..l$,  the $m$-order mutual information
is defined as $I_m(\x_1,\dots,\x_l):= \sum_{i=1}^lh_m(x_i)-h_m(\x_1,\dots,\x_l)$ and the {\em mutual information rate} is defined as the limit \begin{equation}\label{eq:mir}
I_\infty(\x_1,\dots,\x_l):=\lim_{m\to\infty} I_m(\x_1,\dots,\x_l).                                                                                                                                  
\end{equation}

{\noindent\bf Discretisations and a metric.}
For each $m,l\in\N$, let $B^{m,l}$ be  a partitioning of $\X^m$ into $2^l$ sets such that jointly these partitionings generate the Borel $\sigma$-algebra $\mathcal F_m$ of $\X^m$, i.e.\ $\sigma(\cup_{l\in\N}B^{m,l})=\mathcal F_m$. 
%For example, for $m=1$ and $\X=\R$ one can take $B^{m,l}_{1..2^{l-1}$ 

%\begin{definition}[Distributional Distance] \label{defn:dd}
The  distributional distance between a pair of process distributions
$\rho_1,\rho_2$ is defined as follows~\cite{Gray:88}:
\begin{equation}\label{eq:dd}
 d(\rho_1,\rho_2)=\sum_{m,l=1}^\infty w_m w_l \sum_{B\in B^{m,l}} |\rho_1(B)-\rho_2(B)|,
\end{equation}
where we set  $w_j:=1/j(j+1)$, 
 but any summable sequence of positive weights may be used.
%\end{definition}
It is shown in \cite{Ryabko:103s} that empirical estimates of this distance are asymptotically consistent for arbitrary stationary ergodic processes, and these estimates are used in \cite{Ryabko:10clust,Khaleghi:15clust} to construct time-series clustering algorithms for clustering with respect to distribution. Here we will not use this distance in the algorithms, but only in the impossibility results.
Basing on these ideas, Gy\"orfi \cite{Gyorfi:11unp} suggested to use a similar construction for studying  independence, namely
$$
d(\rho_1,\rho_2)=\sum_{m,l=1}^\infty w_m w_l\hskip-4mm \sum_{A,B\in B^{m,l}} |\rho_1(A)\rho_2(B)-\rho(A\times B)|,
$$
where $\rho_1$ and $\rho_2$ are the two marginals of a process  $\rho$ on pairs, and noted that its empirical estimates are asymptotically consistent. Here we will use a similar distance which is based on mutual information instead.

\subsection{Impossibility results}\label{s:no}
First of all, while the definition of ergodic processes (or the ergodic theorem, if one follows the conventional definition)
guarantees   convergence of frequencies to the corresponding probabilities, this convergence can be arbitrary slow \cite{Shields:98}. That is, there 
are no meaningful bounds on $|\nu(X_{1..n},0)-\rho(X_1=0)|$ in terms of $n$,  for ergodic $\rho$.
% This means that all the results one can possibly obtain in this setting are only asymptotic. 
This means that we cannot use tests for (conditional) independence %and mutual information 
 of the kind employed in the i.i.d.\ case (Section~\ref{s:iid}).

Thus, the first question we want to answer is whether it is possible to test independence, that is, to say whether $\x_1\perp\x_2$ based on a stationary ergodic samples $X^1_{1..n}, X^2_{1..n}$.  Here we shall show that the answer in a certain sense is negative, but some important questions remain open. 

A test (for independence) $\phi$  is a function that takes two samples  $X^1_{1..n}, X^2_{1..n}$ and  a parameter $\alpha\in(0,1)$, called the {\em confidence level}, and outputs a binary answer: the samples are independent or not.

%  A test $\phi$ is {\em asymptotically consistent} if $\phi(X^1_{1..n}, X^2_{1..n})$ converges a.s.\ to 0 if the distributions $P_1,P_2$ of  $X^1_{1..n}, X^2_{1..n}$ are independent and to 1 otherwise, for every stationary ergodic $P_1,P_2$. 
% 
% The first result is that this is not possible.
% \begin{proposition}\label{p1}
%  There is no asymptotically consistent test for independence. 
% \end{proposition}
% 
% One subtlety here is that, while we have required the distributions to be stationary and ergodic, we have not required them to be {\em jointly} ergodic, and this is, in fact, the case we are interested in. Here we can show that there is no consistent test for independence for a somewhat stronger notion of consistency, which is as follows (here we use the optional input $\alpha$).

 A test $\phi$ is {\em $\alpha$-level  consistent}  if, for every stationary ergodic distribution $\rho$ over a pair of samples $(X^1_{1..n..}, X^2_{1..n..})$, for every confidence level $\alpha$,  $\rho(\phi_\alpha(X^1_{1..n}, X^2_{1..n})=1)<\alpha$ if the marginal distributions of the samples are independent, and  $\phi_\alpha(X^1_{1..n}, X^2_{1..n})$ converges to~1 as $n\to\infty$ with $\rho$-probability~1 otherwise. 

The following  can be established  thanks to the criterion for the existence of such tests obtained in~\cite{Ryabko:121c}. Recall that, for $\rho\in\mathcal S$, the measure $W_\rho$ over $\mathcal E$ is its ergodic decomposition. The criterion of  \cite{Ryabko:121c} states that there is an $\alpha$-level consistent test for a hypothesis $H_0$ against $\mathcal E\setminus H_0$ if an only if for every $\rho\in\cl H_0$ we have $W_\rho(H_0)=1$.

\begin{proposition}\label{th:notest}
 There is no $\alpha$-level consistent  test for  independence (for jointly stationary ergodic samples). 
\end{proposition}
\begin{proof}

The example is based on the so-called translation process, which is constructed as follows. 
Fix some irrational  $\alpha \in (0,1)$ and select $r_0 \in [0,1]$ uniformly at random. 
For each $i=1..n..$ let $r_i=(r_{i-1}+\alpha) \mod 1$ (that is, the previous element is shifted by $\alpha$ to the right, considering the [0,1] interval looped). 
The samples $X_i$ are obtained from $r_i$ by thresholding at $1/2$, 
i.e. $X_i:=\mathbb{I}\{r_i>0.5\}$ (here $r_i$ can be considered hidden states). This process is stationary and ergodic; besides, it has 0 entropy rate~\cite{Shields:98}, and this is not the last of its peculiarities. 

Take now two independent copies of this process to obtain a pair $(\x_1,\x_2)=(X_1^1,X_1^2\dots,X_n^1,X_n^2,\dots)$. %and  $\x_2=(X_1^2,\dots,X_n^2,\dots)$.
 The resulting process on pairs, which we denote $\rho$, is stationary, but it is not ergodic. To see the latter, observe that the difference between the corresponding hidden states remains constant. 
In fact, each initial state $(r_1,r_2)$ corresponds to an ergodic component of our process on pairs. By the same argument, these ergodic components are not independent. Thus, we have taken two independent copies of a stationary ergodic process, and obtained a stationary process which is not ergodic and whose ergodic components are pairs of processes that are not independent!

To apply the criterion cited above, it remains to show that the process $\rho$ we constructed can be obtained as a limit of stationary ergodic processes on pairs. To see this, consider, for each $\epsilon$,  a process $\rho_\epsilon$, whose construction is identical to $\rho$ except that instead of shifting the hidden states by $\alpha$ we shift them by $\alpha+u_i^\epsilon$ where $u_i^\epsilon$ are i.i.d.\ uniformly random on $[-\epsilon,\epsilon]$. It is easy to see that $\lim_{\epsilon\to0} \rho_\epsilon=\rho$ in distributional distance, and all $\rho_\epsilon$ are stationary ergodic.  Thus, if $H_0$ is the set of all stationary ergodic distributions on pairs, we have found a distribution $\rho\in\cl H_0$ such that $W_\rho(H_0)=0$. We can conclude that  there is no $\alpha$-level consistent test for $H_0$ against its complement. 
\end{proof}

Thus, there is no consistent test that could provide a given level of confidence under $H_0$, even if only asymptotic consistency is required under $H_1$. However, a yet weaker  notion of consistency might  suffice to construct asymptotically consistent clustering algorithms. Namely, we could ask for a test whose answer converges to either 0 or 1 according to whether the distributions generating the samples are independent or not. Unfortunately, we do not know whether a test consistent in this weaker sense  exists or not.  We conjecture that it does not. The conjecture is based not only on the result above, but also on the result of \cite{Ryabko:10discr} that shows that there is no such test for the related problem of homogeneity testing, that is, for testing whether  two given samples have the same or different distributions. This negative result holds even if the distributions are independent, binary-valued, the difference is restricted to $P(X_0=0)$, and, finally, they have to be $B$-processes (a family of distributions much smaller than that of all stationary ergodic ones).

Thus, for now what we can say is that there is no test for independence  available that would be consistent under ergodic sampling. This means that we cannot distinguish even between  the cases of 1 and 2 clusters, and so we shall only consider  the problem of clustering with the number of clusters $k$ known. 

It is also worth noting that several related seemingly innocuous questions, that have simple answers for i.i.d.\ sampling, remain open for ergodic sampling. For example, for i.i.d.\ processes it is easy to show that $\alpha$-level  consistency is strictly stronger than asymptotic consistency just introduced, meaning that if there is an $\alpha$-level consistent test for a hypothesis $H_0$ against its complement then there exists an asymptotically consistent test for the same hypothesis, and the reverse is not always true. However, for ergodic sampling this remains to be demonstrated. Another open question (posed by \cite{Nobel:06}) is whether one can replace ``with probability 1'' with ``in expectation'' in the definition of asymptotic consistency; again, this is the case for i.i.d.\ (and mixing) distributions (that is, the two resulting notions are equivalent), but for ergodic distributions we do not know. One can pose the same question for $\alpha$-level consistency (where ``w.p.~1'' refers to the convergence of the Type~II error).

Finally, the last problem we will have to address is mutual information for processes. The analogue of mutual information for stationary processes is the mutual information rate~\eqref{eq:mir}. Unfortunately, 0 mutual information rate does not imply independence. This is manifest on processes with 0 entropy rate, for example those of the example in the proof of Proposition~\ref{th:notest}.
What happens is that, if two processes are dependent, then indeed at least one of the $m$-order entropy rates $I_m$ is non-zero, but the limit may still be zero. Since we do not know in advance which $I_m$ to take, we will have to consider all of them, as is explained in the next subsection.

\subsection{Clustering with the number of clusters known}
The  quantity introduced below, which we call sum-information, will serve as an analogue of mutual information in 
the i.i.d.\ case, allowing us to get around the problem that the mutual information rate may be~0 for a pair of dependent stationary ergodic processes. Defined in the same vein as the distributional distance~\eqref{eq:dd}, this new quantity is a weighted sum over all the mutual informations up to time $n$; in addition, all the individual mutual informations are computed for quantized versions of random variables in question, with decreasing cell size of quantization, keeping all the mutual information resulting  from different quantizations. The latter allows us not to require the existence of densities. Weighting is needed in order to be able to obtain consistent empirical estimates of the theoretical quantity under study.

First, for a process $\x=(X_1,\dots,X_n,\dots)$ and for each $m,l\in\N$ define the $l$'th quantized version $[X_{1..m}]^l$  of $X_{1..m}$ 
as the index of the cell of $B^{m,l}$ to which $X_{1..m}$ belongs. %: $[X_{1..m}]^l:=i: X_{1..m}$
Recall that each of the partitions $B^{m,l}$ has cell size  $2^l$, and that $w_l:=1/l(l+1)$.
Note that, compared to the distributional distance~\eqref{eq:dd}, here we have additional scaling weights that are needed to bring the summands into $[0,1]$.
\begin{definition}[sum-information] For stationary processes $x_1,\dots,x_k$ define the sum-information
 \begin{multline}\label{eq:smir}
  \I(\x_1,\dots,\x_N):=\sum_{m=1}^\infty {1\over m} w_m \sum_{l=1}^\infty  {1\over l} w_l
  \\ 
  \left(\sum_{i=1}^N h([X^i_{1..m}]^l)\right) - h ([X^1_{1..m}]^l,\dots,[X^N_{1..m}]^l)
 \end{multline}
\end{definition}
A somewhat similar device is used in \cite{BRyabko:09} for the purpose of density estimation. Note, however, that $\I$ is not an estimator, but a theoretical quantity which we will be, in fact,  estimating empirically.

The following statement is easy to see, based on the fact that $\cup_{l\in\N}B^{m,l}$ generates $\mathcal F_m$ and $\cup_{m\in\N}\mathcal F_m$ generates $\mathcal F_\infty$.
\begin{lemma}\label{th:ii}
 $\I(\x_1,\dots,\x_N)=0$ if  and only if $\x_1,\dots,\x_N$ are mutually independent.
\end{lemma}

\begin{definition}[Empirical estimates: $\hat h, \hI$]
 Empirical estimates of entropy are defined simply by replacing unknown probabilities by frequencies:
$$
\hat h_n([X^i_{1..m}]^l):=-\sum_{B\in B^{m,l}}\nu(X_{1..n},B)\log\nu(X_{1..n},B),
$$
and likewise for the multivariate versions. 
The { empirical estimate} $\hI_n(\x_1,\dots,\x_N)$ of $\I(\x_1,\dots,\x_N)$ is obtained by replacing the entropies in~\eqref{eq:smir} by their empirical estimates. % (here as before  $n$ is the size of the shortest sample  $n:=\min_{i=1..N}n_i$). 
\end{definition}
Note that the usual Laplace or Krichevsky-Trofimov  corrections for $0$ frequencies are not necessary, since we are not measuring the KL divergence w.r.t.\ the true distribution.
\begin{remark}[\bf Computing $\hI_n$]\label{r:comp} The expression~\eqref{eq:smir} might appear to hint at a computational disaster, 
as it involves two infinite sums, and, in addition, the number of elements in the sum inside $h([]^l)$ grows exponentially in $l$. However, it is easy to see that, when we replace the probabilities with frequencies, all but  a finite number of summands are either zero or can be collapsed (because they are constant). Moreover, the sums can be further truncated  so that the total computation becomes quasilinear in $n$. This can be done exactly the same way as for the distributional distance, as described in detail in \cite[Section 5]{Khaleghi:15clust}. %, so  the argument is not reproduced here.
\end{remark}

\begin{lemma}\label{th:hi}
 Let the distribution $\rho$ of $\x_1,\dots,\x_N$ be jointly stationary ergodic. Then $\hI_n(\x_1,\dots,\x_k)\to \I(\x_1,\dots,\x_N)$ $\rho$-a.s.
\end{lemma}
\begin{proof}[Proof idea]
The lemma can be proven analogously to the corresponding statement about consistency of empirical estimates of the distributional distance, given in \cite[Lemma 1]{Ryabko:103s}. The main idea is that each frequency is an asymptotically consistent estimate of the corresponding probability. For each sample size $n$ we do not know which of the estimates are already within, say, $\epsilon$ of the limit. However, for each $\epsilon$ we can find a {\em finite} $M,L$ such that the combined weight of all $m>M, l>L$ is less than $\epsilon$; we can then find a sample size $n$ such that from $n$ on all of the estimates with $m\le M,l\le L$ are within $\epsilon/ML$ of the limit.
\end{proof}

This lemma alone is enough to establish the existence of a consistent clustering algorithm. To see this, first consider the following problem, which is the ``independence'' version of the classical statistical three-sample problem.

{\noindent\bf The 3-sample-independence} problem. Three samples $\x_1,\x_2,\x_3$,  are given, and it is known that {\em either} $(\x_1,\x_2)\perp\x_3$ {\em or} $\x_1\perp(\x_2,\x_3)$ but not both. It is required to find out which one is the case.
\begin{proposition}\label{th:3s} 
  There exists an algorithm for solving the 3-sample-independence problem that is asymptotically consistent under ergodic sampling.
\end{proposition}
\begin{proof}
 The algorithm compares $\hI_n((\x_1,\x_2),\x_3)$ and  $\hI_n(\x_1,(\x_2,\x_3))$ and answers according to whichever is smaller. From the consistency of $\hI_n$ (Lemma~\ref{th:hi}) it follows that,  w.p.~1 from some $n$ on the answer is correct.
\end{proof}

The independence clustering problem which we are after is a generalisation of the 3-sample-independence problem to $N$ samples. We can also have a consistent algorithm for the clustering problem, simply comparing all possible clusterings $U_1,\dots,U_k$ of the $N$ samples 
given and selecting whichever minimizes $\hI_n(U_1,\dots,U_k)$. While this is already fine from the statistical point of view, such an algorithm is of course not practical, since the number of computations it makes must be exponential in $N$, as well as in $k$. It turns out that one can reduce the number of candidate clustering dramatically, making the problem amenable to computation.

The proposed algorithm CLINk (see pseudocode in  Algorithm~\ref{alg:2} below) works similarly to the algorithm CLIN for known distributions and i.i.d.\ sampling, but with some important differences. Like before, the main procedure is to attempt to split the given set of samples into two clusters. This splitting procedure starts  with a single element $\x_1$ and estimates its sum-information $\hI(\x_1,R)$ with the rest of the elements, $R$. It then takes the elements out of $R$ one by one without replacement, measuring each time how this  changes $\hI(\x_1,R)$. As before, once and if we find an element that is not independent of $\x_1$, this change will be positive. However, unlike in the   previous case (in CLIN), here we cannot test whether this change is  0 or not. Yet, we can say that {\em if}, among the tested elements, there is one that gives a non-zero change in $\I$, then one of such elements will be the one that gives the maximal change in $\hI$ (provided, of course, that we have enough data for the estimates $\hI$ to be close enough to the theoretical values $\I$). 
We therefore keep each split that arises from such {\em maximal-change} element, resulting in $O(N^2)$ candidate splits for the case of 2 clusters. For $k$ clusters, we have to consider all the combinations of the splits, resulting in $O(N^{2k-2})$ candidate clusterings. We then simply select the clustering that minimizes $\hI$ among the candidate clusters. Asymptotic consistency then follows from the asymptotic consistency of the estimates, as is shown below.

\begin{algorithm}[h]
\caption{CLINk: cluster given $k$ and an estimator of mutual sum-information}
\label{alg:2}
\begin{algorithmic}
% \STATE {INPUT: The set $S$.}
% \STATE {$(C_1,C_2):=Split(S)$}
% \IF {$C_2\ne\empty$}
%  \STATE{ Output:$CLIN(C_1),CLIN(C_2)$}
%  \ELSE {\STATE{ Output: $C_1$}}
% \ENDIF
\STATE{ Consider all $N^{k-1}$ clusterings obtained by applying recursively the function Split to each of the sets in each of the candidate partitions, starting with the input set $S$, until $k$ clusters are obtained. Output the clustering $U$ that minimizes $\hI(U)$}

{\bf Function {Split}}{(Set $S$ of samples)}%{$\{\x_1,\dots,\x_m\}$} 
\STATE {Initialize: $C:=\{\x_1\},$ $R:=S\setminus C,$ $\mathcal P:=\{\}$} %$S:= \{\x_2,\dots,\x_m\}$ }%  $S\setminus C$ }%\{\x_2,\dots,\x_m\}$}
 \WHILE {$R\ne\emptyset$}
 \STATE {Initialize:$M:=\{\}$,\ % R:=S\setminus C$,
  $d:=0$;\\\hskip12mm\text{xmax:= index of any $\x$ in $R$}}
 \STATE{Add $(C,R)$ to $\mathcal P$}
\FOR {each $\x\in R$}
  \STATE{$r:=\hat\I(C,R)$}
 \STATE{move $\x$ from $R$ to $M$}
  \STATE{$r':=\hat\I(C,R)$; $d':=r-r'$ }
 \IF {$d'>d$} \STATE{$d:=d', \text{xmax:=index of(}\x)$}
 \ENDIF
\ENDFOR
 \STATE{Move $\x_{xmax}$ from $M$ to $C$;  $R:=S\setminus C$ }
\ENDWHILE
% \STATE { $O:=O\backslash M$}
\STATE{Return(List of candidate splits $\mathcal P$)}
%\IF {$R\ne\emptyset$} \STATE Return($C,R$) \ELSE \STATE Return($C$)\ENDIF
{\\\bf END function}
\end{algorithmic}
\end{algorithm}
\begin{theorem} The output of the CLINk algorithm is  asymptotically consistent under ergodic sampling. This algorithm makes 
% correct clustering using
 at most $N^{2k-2}$ calls to the estimator of mutual sum-information.
\end{theorem}
Note that, as follows from Remark~\ref{r:comp}, the total amount of computation required to run the algorithm is polynomial in each of the rest of the parameters.
\begin{proof}
 The consistency of $\hI$ (Lemma~\ref{th:hi}) implies that, for every $\epsilon>0$, from some $n$ on w.p.~1, all the estimates of $\I$ the algorithm uses will be within $\epsilon$ of the corresponding $\I$ values. Since $I(U_1,\dots,U_k)=0$ if and only if $U_1,\dots,U_k$ is the correct clustering (Lemma~\ref{th:ii}), it is enough to show that, assuming all the $\hI$ estimates are close enough to the $\I$ values, the clustering that minimizes $\hI(U_1,\dots,U_k)$ is among those the algorithm searchers through, that is, among the clusterings obtained by applying recursively the function Split to each of the sets in each of the candidate partitions, starting with the input set $S$, until $k$ clusters are obtained.

To see the latter, on each iteration of the \texttt{while} loop, we either already have a correct candidate split in $\mathcal P$, that is, a split $(U_1,U_2)$ such that $\I(U_1,U_2)=0$, or we find (executing the \texttt{for} loop) an element $\x'$ to add to the set $C$ such that $C\nperp {\x'}$.
Indeed, if at least one such  element $\x'$ exists, then among all such elements there is one that maximizes the difference $d'$.  Since the set $C$ is initialized as a singleton, a correct split is eventually found if it exists. Applying the same procedure exhaustively to each of the elements of each of the candidate splits producing all the combinations of $k$ candidate clusterings, under the assumption that all the estimates $\hI$ are sufficiently close the corresponding values, we are guaranteed to have the one that minimizes $I(U_1,\dots,U_k)$ among the output.
\end{proof}

\begin{remark}[{\bf Fickle oracle}]\label{r:fickle}
 Another way to look at the difference between the stationary and the i.i.d.\ cases is to consider the following ``fickle'' version of the oracle test of Section~\ref{s:know}. Consider the oracle that, as before, given sets of random variables $A,B,C,D\subset\{\x_1,\dots,\x_N\}$ answers whether $\I(A,B)>\I(C,D)$. However, the answer is only guaranteed to be correct in the case $\I(A,B)\ne \I(C,D)$. If $\I(A,B)=\I(C,D)$ then the answer is arbitrary (and can be considered adversarial). One can see that Lemma~\ref{th:hi} guarantees the existence of the oracle that has the requisite fickle correctness property asymptotically, that is, w.p.~1 from some $n$ on. It is also easy to see that Algorithm~\ref{alg:2} can be rewritten in terms of calls to such an oracle.
\end{remark}

\section{Generalizations, future work}\label{s:ext}
A general formulation of the independence clustering problem has been presented, and attempt has been made to  trace out broadly the limits of what is possible and what is not possible in this formulation.  In doing so,    clear-cut formulations have been favoured over utmost generality, and over, on the other end of the spectrum, precise performance guarantees.  
Thus, many interesting questions have necessarily been left out, and some of these constitute exciting directions for further research which  are outlined  in this section.
{\\\noindent\bf Beyond time series.} For the case when the distribution of the random variables $\x_i$ is unknown, we have assumed that a sample $X_{1..n}^i$ is available for each $i=1..N$. Thus, each $\x_i$ is represented by a time series.  A time series is but one form the data may come in. Other ways include functional data, mutli-dimensional- or continuous-time processes,  or graphs.  Generalizations to some of these models, such as, for example, space-time stationary processes, are relatively straightforward, while others require more care. In either case, the problem is statistical (rather than algorithmic).
%If we want  a clustering algorithm with $k$ unknown,  
We need to be able to replace the emulate the oracle test of section~\ref{s:know} with statistical tests, if we want to be able to find the correct clustering with $k$ unknown. As explained in Section~\ref{s:iid}, it is sufficient to find a test for conditional independence, or an estimator of entropy along with guarantees on its convergence rates. If these are not available, as is the case of stationary ergodic samples, we can still have a consistent algorithm for $k$ known as long as we have an asymptotically consistent estimator of mutual information (without rates), or, more generally, can emulate a fickle oracle (Remark~\ref{r:fickle}).  
% Note that different sampling models may require a non-trivial generalisation of mutual information; however, even for time series, we are not restricted to Shannon mutual information, but can use any other divergence that can accurately discern independence. 
{\\\noindent\bf Beyond independence.} The problem formulation considered rests on the assumption that there exists a partition $U_1,\dots,U_k$ of the input set $S$ such that $U_1,\dots,U_k$ are jointly independent, that is, such that $I(U_1,\dots,U_k)=0$. In reality, perhaps, nothing is really independent, and so some relaxations are in order. It is easy to introduce some thresholding in the algorithms (replacing 0 in each test by some threshold $\alpha$) and derive some basic consistency guarantees for the resulting algorithms. The general problem formulation is to find a finest clustering such that $I(U_1,\dots,U_k)>\epsilon$, for a given $\epsilon$ (note that, unlike in the independence  case of $\epsilon=0$, such a clustering may not be unique). If one wants to get rid of $\epsilon$, a tree of clusterings may be considered for all $\epsilon\ge0$, which is a common way to treat unknown parameters in the clustering literature (e.g.,\cite{balcan2014robust}). 
%This hierarchical problem (and its complexity) is very interesting, but so far we leave it for future work.

Another approach to generalization comes from considering the problem from the graphical model point of view. The random variables $\x_i$ are vertices of a graph, and edges represent dependencies. In this representation, clusters  are connected components of the graph. A generalization then is then to clusters that are the smallest components that are connected (to each other) by at most $l$ edges, where $l$ is a parameter. 

 Yet another generalization would be  to decomposable distributions of \cite{jirouvsek1991solution}.
{\\\noindent\bf Different assumptions on  time series.} Independence and stationarity are general, qualitative conditions, whose validity  one may reasonably hope to be able to judge for each application based on general considerations.  This makes these conditions suitable for problems where little is known about the data at hand, such as the clustering problem.  However, these conditions may not be entirely satisfying in applications. For example, the i.i.d.\ assumption may be clearly inapplicable, while stationarity may be too weak, since it does not allow us to find the number of clusters. In such cases one may resort to other assumptions.  Mixing conditions (e.g., \cite{Bosq:96,Rio:13}) provide one avenue of research. These conditions generalize the i.i.d.\ assumption, carrying over many of the corresponding concentration inequalities, while allowing some (limited) dependence into the picture. One can thus expect the results of Section~\ref{s:iid} to be generalizable to the case of mixing time series, again, by constructing appropriate tests of conditional independence or analysing  estimators of the mutual information.  On the other end of the generality spectrum, one may wish to generalize stationarity. One rather easy generalization is to that asymptotic mean stationary processes \cite{Gray:88}. It is, in fact, easy to see that the results of Section~\ref{s:st} carry over to this case.
{\\\noindent\bf Performance guarantees.}
Non-asymptotic results (finite-sample performance guarantees) can be obtained under additional assumptions, using the corresponding results on (conditional) independence tests and on estimators of divergence between distributions. Here it is worth noting  that we are not restricted to using the mutual information $I$, but any measure of divergence can be used, for example, R\'enyi divergence; a variety of relevant estimators and corresponding bounds, obtained under such assumptions as H\"older continuity, can be found in \cite{pal2010estimation,kandasamy2014influence}. From any such bounds, performance guarantees for CLIN can be obtained simply using the union bound over all the invocations of the tests. A potential avenue of research is trying to find out whether this latter step is the best one can do~--- probably not, since many of the tests are on very similar sets of data. Besides, optimizing the number of tests (see the next paragraph) can also lead to improvement in performance guarantees. At the same time, as mentioned in Section~\ref{s:st}, for stationary ergodic data, in general, there are (provably) no non-trivial finite-time performance guarantees.
{\\\noindent\bf Complexity.} While all the questions discussed so far are statistical, there is also an intriguing algorithmic question remaining, which arises already for the case of known distributions: the computational complexity of the problem. We have presented  upper bounds, by constructing algorithms and bounding their  complexity ($kN^2$ for CLIN and $N^{2k}$ for CLIN$k$), which shows  that  all the algorithms are computationally feasible, but little beyond that. For the case of unknown $k$, it is clear that one cannot do with less than $O(N^2)$ computations in general; beyond these considerations,   the complexity of each of the problems is left for future work.
 A subtlety worth noting is that, for the case of known distributions, the complexity may be affected by the choice of the oracle. In other words, some calculations may be ``pushed'' inside the oracle. In this regard, it may be better to consider the oracle for testing conditional independence, rather than a comparison of mutual informations, as explained in Remarks~\ref{r:or}, \ref{r:et}. 
%; or else, to consider the problem directly for unknown distributions, under either  i.i.d.\ or ergodic sampling. 
%, even though it has more parameters (notably the sample length $n$).
 
The complexity of the stationary-sampling version of the problem can be studied using the {\em fickle oracle} of Remark~\ref{r:fickle}. The consistency of the algorithm should then be established for {\em every} assignment of those answers of the oracle that are arbitrary (adversarial). % (as is done  in Section~\ref{s:st}).
 Considering such an oracle allows one to separate completely the statistical problem from the algorithmic one. Such a separation resolves the problem mentioned in \cite{Khaleghi:15clust}, namely, that it does not make sense to attempt to minimize the computational complexity of the algorithm requiring asymptotic consistency alone: indeed, to reduce the complexity, one could always through away a portion of the data, considering, for example, the first $\log n$ elements of every time series (of length $n$).
%  Separating the statistical part of the problem into the problem of constructing an appropriate oracle, be it the one of Section~\ref{s:know} or the fickle one, allows one to avoid such problems.
{\\\noindent\bf Experiments.} Experimental evaluations of the proposed algorithms in different applications domains presents a fascinating direction for future research. Here it is worth repeating that, while the estimator of Section~\ref{s:st} may seem computationally infeasible, in fact, it is not, as it may be computed similarly to the distributional distance estimator as is done in~\cite{Khaleghi:15clust}. However, in this respect other measures of dependence may become attractive, even if they lack some of the theoretical guarantees. Thus, mutual information {\em rate} $I_\infty$ for stationary ergodic discrete-valued time series may be estimated using data compressors, in the spirit of  \cite{Cilibrasi:05,BRyabko:06a}, and then used in CLIN$k$ instead of $\hI$. Specifically, assuming discrete-valued time series, given a data compressor $\phi$ and  time series  $X^i_{1..n},$ $i=1..N$,  one can use the estimator $\hat I_\infty(X^1_{1..n},\dots,X^N_{1..n})=\sum_{i=1}^N|\phi(X_{1..n})|-|\phi((X^1_{1},\dots,X^N_1),\dots,(X^1_{n},\dots,X^N_n))|$, where $|\cdot|$ denotes the length.
Of course, using mutual information rate instead of, say, $\hI$,  breaks consistency in general (as explained in Section~\ref{s:st}), but the potential of harnessing already available compressing methods is nevertheless appealing.
{\\\noindent\bf (Im)possibility results for stationary sampling.} Finally, it is worth mentioning that hypothesis testing  for stationary ergodic time series is a fascinating field most of which remains to be explored. Some of the questions most pertinent to the considered problem have been pointed out in Section~\ref{s:no}. The main open question is characterizing those (composite) hypotheses for which consistent tests exist, for various notions of consistency, as well as relating the different notions of consistency to each other.  

% threshold dependence; more general than time series - space-time st., graphs, functional data; other forms of sampling - mixing, AMS, etc.;
% \section{Discussion and future work}
% bounds: test is not estimation; sets are oeverlapping ; computational complexity: lower bounds, the oracl is important; Impossibility vs possibility for st erg 

% 
% % 
%  \bibliographystyle{plain}
%  \bibliography{../my}

\end{document}